\documentclass[journal]{IEEEtran}

\usepackage{microtype}
\usepackage{graphicx}
\usepackage{booktabs} 
\usepackage{xcolor}
\usepackage{subcaption}
\usepackage{hyperref}
\usepackage{multirow}

\usepackage{amsthm}
\usepackage{amsfonts}
\usepackage{thmtools}
\usepackage{thm-restate}
\usepackage{mathtools}

\declaretheorem[name=Definition]{definition}

\newcommand \TODO [1] {}

\newcommand{\E}{\mathbb E}

\newcommand \mc \mathcal

\newcommand {\augm} [1] { \pmb{{#1}}}  

\newcommand {\x} {\augm x}

\newcommand {\hide} [1] {}

\pdfoutput=1

\ifCLASSINFOpdf
\else
\fi
\usepackage{algorithm}
\usepackage{algorithmic}
\begin{document}
%
\title{Delay-Aware Model-Based Reinforcement Learning for Continuous Control}
%
%
%


\author{Baiming~Chen, Mengdi~Xu, Liang~Li, Ding~Zhao*
\thanks{Baiming Chen is with the State Key Laboratory of Automotive Safety and Energy, Tsinghua University, Beijing 100084, China (e-mail:cbm17@mails.tsinghua.edu.cn).}
\thanks{Mengdi~Xu is with the Department of Mechanical Engineering, Carnegie Mellon University, Pittsburgh, PA 15213, USA (e-mail:mengdixu@andrew.cmu.edu).}
\thanks{Liang Li is with the State Key Laboratory of Automotive Safety and Energy, Tsinghua University, Beijing, China, and is also with Collaborative Innovation Center of Electric Vehicles in Beijing 100084, China (e-mail:liangl@tsinghua.edu.cn).}
\thanks{Ding Zhao is with the Department of Mechanical Engineering, Carnegie Mellon University, Pittsburgh, PA 15213, USA (e-mail:dingzhao@cmu.edu).}
}

\maketitle

\begin{abstract}
Action delays degrade the performance of reinforcement learning in many real-world systems. This paper proposes a formal definition of delay-aware Markov Decision Process and proves it can be transformed into standard MDP with augmented states using the Markov reward process. We develop a delay-aware model-based reinforcement learning framework that can incorporate the multi-step delay into the learned system models without learning effort. Experiments with the Gym and MuJoCo platforms show that the proposed delay-aware model-based algorithm is more efficient in training and transferable between systems with various durations of delay compared with off-policy model-free reinforcement learning methods. Codes available at: https://github.com/baimingc/dambrl.
\end{abstract}


\IEEEpeerreviewmaketitle

\section{Introduction}
Deep reinforcement learning has made rapid progress in games \cite{mnih2013playing, silver2016mastering} and robotic control \cite{schulman2015trust, duan2016benchmarking, hwangbo2017control}. However, most algorithms are evaluated in turn-based simulators like Gym \cite{brockman2016openai} and MuJoCo \cite{todorov2012mujoco}, where the action selection and actuation of the agent are assumed to be instantaneous. Action delay, although prevalent in many areas of the real world, including robotic systems \cite{imaida2004ground, jin2008robust, bayan2009brake}, communication networks \cite{moon1999estimation} and parallel computing \cite{hannah2018unbounded}, may not be directly handled in this scheme. 

Previous research has shown that delays would not only degrade the performance of the agent but also induce instability to the dynamic systems \cite{gu2003survey, dugard1998stability, chung1995time}, which is a fatal threat in safety-critical systems like connected and autonomous vehicles (CAVs) \cite{gong2016constrained}. For instance, it usually takes more than 0.4 seconds for the hydraulic automotive brake system to generate the desired deceleration \cite{bayan2009brake}, which could make a huge impact on the planning and control modules of CAVs \cite{ploeg2013lp}. The control community has proposed several methods to address this problem, such as using Smith predictor \cite{astrom1994new, matausek1999modified}, Artstein reduction \cite{artstein1982linear, moulay2008finite}, finite spectrum assignment \cite{manitius1979finite, mondie2003finite}, and $H_{\infty}$ controller \cite{mirkin2000extraction}. Most of these methods depend on accurate models \cite{niculescu2001delay, gu2003survey}, which is usually not available in the real-world applications.

Recently, DRL has offered the potential to resolve this issue. The problems that DRL solves are usually modeled as Markov Decision Process (MDP).
However, ignoring the delay of agents violates the Markov property and results in partially observable MDPs, or POMDPs, with historical actions as hidden states. From \cite{singh1994learning}, it is shown that solving POMDPs without estimating hidden states can lead to arbitrarily suboptimal policies. To retrieve the Markov property, the delayed system was reformulated as an augmented MDP problem such as the work in \cite{katsikopoulos2003markov, walsh2009learning}. While the problem was elegantly formulated, the computational cost increases exponentially as the delay increases.
Travnik \textit{et al.} \cite{travnik2018reactive} showed that the traditional MDP framework is ill-defined, but did not provide a theoretical analysis.
Ramstedt \& Pal \cite{ramstedt2019real} proposed an off-policy model-free algorithm known as Real-Time Actor-Critic to address the delayed problem by adapting Q-learning to state-value-learning.
The delay issue could also be relieved with the model-based manner by learning a dynamics model to predict the future state as in \cite{walsh2009learning}. However, this paper mainly focused on discrete tasks and could suffer from the curse of dimensionality when discretizing state and action space for continuous control tasks \cite{lillicrap2015continuous}.

In this paper, we further explore reinforcement learning methods on delayed systems in the following three aspects: 1) We formally define the multi-step delayed MDP and prove it can be converted to standard MDP via the Markov reward process.
2) We propose a general framework of delay-aware model-based reinforcement learning for continuous control tasks.
3) By synthesizing the state-of-the-art modeling and planning algorithms, we develop the Delay-Aware Trajectory Sampling (DATS) algorithm which can efficiently solve delayed MDPs with minimal degradation of performance.

The rest of the paper is organized as follows. We first review the preliminaries in Section~\ref{sec:pre} including the definition of Delay-Aware Markov Decision Process (DA-MDP). In Section~\ref{sec:del}, we formally define the Delay-Aware Markov Reward Process (DA-MRP) and prove its solidity. In Section~\ref{sec:dambrl}, we introduce the proposed framework of delay-aware model-based reinforcement learning for DA-MDPs with a concrete algorithm: Delay-Aware Trajectory Sampling (DATS). In Section~\ref{sec:exp}, we demonstrate the performance of the proposed algorithm in challenging continuous control tasks on Gym and MuJoCo platforms.

\section{Preliminaries}
\label{sec:pre}
\subsection{Delay-Free MDP and Reinforcement Learning}

The Delay-free MDP framework is suitable to model games like chess and go, where the state keeps still until a new action is executed. The definition of a delay-free MDP is:

\begin{definition}
\label{def:MDP}
A Markov Decision Process (MDP) is characterized by a tuple with \\
(1) state space $\mathcal{S}$, \hspace{0.15cm} (2) action space $\mathcal{A}$, \hspace{0.15cm} \\
(3) initial state distribution $\rho: \mathcal{S} \to \mathbb R$,\\
(4) transition distribution {$p: \mathcal{S} \times \mathcal{S} \times \mathcal{A} \to \mathbb R$}, \hspace{0.15cm}\\
(5) reward function $r: \mathcal{S} \times \mathcal{A} \to \mathbb R$.
\end{definition}

In the framework of reinforcement learning, the problem is often modeled as an MDP, and the agent is represented by a policy $\pi$ that directs the action selection, given the current observation.
The objective is to find the optimal policy $\pi^*$ that maximizes the expected cumulative discounted reward $\Sigma_{t=0}^{T}\gamma^tr\left(s_t, a_t\right)$. 
Throughout this paper, we assume that we know the reward function $r$ and do not know the transition distribution $p$.

\subsection{Delay-Aware MDP}

The delay-free MDP is problematic with agent delays and could lead to arbitrarily suboptimal policies \cite{singh1994learning}. To retrieve the Markov property, Delay-Aware MDP (DA-MDP) is proposed:
\begin{definition} \label{def:DMDP}
A Delaye-Aware Markov Decision Process $D\!A\!M\!D\!P(E, n)=(\pmb{\mathcal{X}}, \pmb{\mathcal{A}}, \augm \rho, \augm p, \augm r)$ augments a Markov Decision Process $M\!D\!P(E) = (\mathcal{S}, \mathcal{A}, \rho, p, r)$, such that \\
(1) state space $\pmb{\mathcal{X}} = \mathcal{S} \times \mathcal{A}^{n}$ where $n$ denotes the delay step, \\
(2) action space $\pmb{\mathcal{A}} = \mathcal{A}$, \\
(3) initial state distribution
\begin{equation*}
\begin{aligned}
\augm \rho(\x_0) = 
\augm \rho({s_0, a_0, \dots, a_{n-1}}) = \rho(s_0) \ \prod_{i=0}^{n-1}\delta(a_i - c_i),
\end{aligned}
\footnote{$\delta$ is the Dirac delta function. If $y \sim \delta(\cdot - x)$ then $y=x$ with probability one.}
\end{equation*}
where $(c_i)_{i=1:n-1}$ denotes the initial action sequence, \\
(4) transition distribution
\begin{equation*}
\begin{aligned}
&\augm p(\x_{t+1}|\x_t, \augm a_t) \\
&= 
\augm p({s_{t+1}, a_{t+1}^{(t+1)}, \dots, a_{t+n}^{(t+1)}} | {s_t, a_t^{(t)}, \dots, a_{t+n-1}^{(t)} }, \augm a_t) \\
&= p(s_{t+1} | s_t, a_t^{(t)}) \prod_{i=1}^{n-1}\delta(a_{t+i}^{(t+1)} - a_{t+i}^{(t)}) \delta(a_{t+n}^{(t+1)} - \augm a_t),
\end{aligned}
\end{equation*}
(5) reward function 
\begin{equation*}
\begin{aligned}
\augm r(\x_t, \augm a_t) = 
\augm r({s_t, a_t, \dots, a_{t+n-1}}, \augm a_t) = r(s_t, a_t).
\end{aligned}
\end{equation*}
The state vector of DA-MDP is augmented with an action sequence 
being executed in the next $n$ steps where $n \in \mathbb{N}$ is the delay duration.
The superscript of $a_{t_1}^{(t_2)}$ means that the action is one element of $\x_{t_2}$ and the subscript represents the action executed time.
$\augm a_t$ is the action taken at time $t$ in a DA-MDP but executed at time $t+n$ due to the $n$-step action delay, i.e. $\augm a_t= a_{t+n}$.
\end{definition}

Policies interacting with the DA-MDPs, which also need to be augmented since the dimension of state vectors has changed, are denoted by bold $\augm \pi$. Fig.~\ref{fig:mdps}, which compares MDP and DA-MDP, shows that the state vector of DA-MDP is augmented with an action sequence 
being executed in the next $n$ steps.

\begin{figure}[t]
\centering
\vskip 0.2in
\begin{subfigure}{0.33\textwidth}

\includegraphics[width=\linewidth]{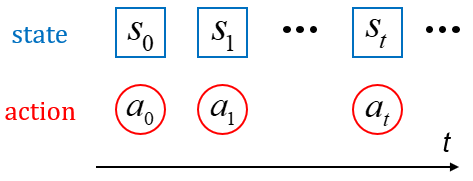} 
\caption{$M\!D\!P(E)$}
\label{fig:mdp}
\end{subfigure}

\begin{subfigure}{0.33\textwidth}
\includegraphics[width=\linewidth]{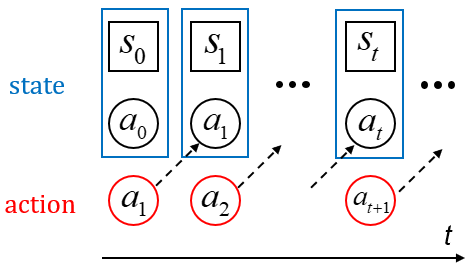}
\caption{$D\!A\!M\!D\!P(E, 1)$}
\label{fig:dmdp}
\end{subfigure}

\begin{subfigure}{0.33\textwidth}
\includegraphics[width=\linewidth]{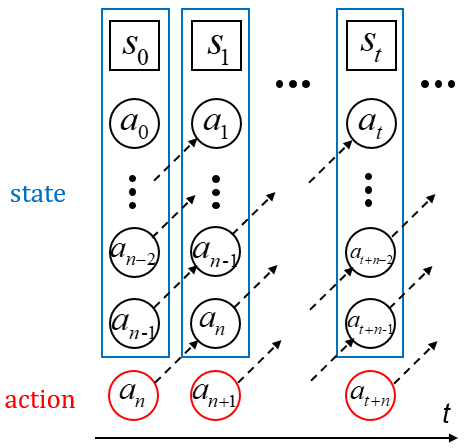}
\caption{$D\!A\!M\!D\!P(E, n)$}
\label{fig:dmdpn}
\end{subfigure}
\caption{Comparison between $M\!D\!P(E)$, $D\!M\!D\!P(E, 1)$ and $D\!M\!D\!P(E, n)$. $n \in \mathbb{N}$ denotes the action delay step. $s_t$ denotes the observed state while $a_t$ denotes the action executed, both at time $t$. Arrows represent how the action selected in the current time step will be included in the future state.}
\label{fig:mdps}
\vskip -0.2in
\end{figure}

It should be noted that both action delay and observation delay could exist in real-world systems. However, it has been proved that from the point of view of the learning agent, observation and action delays form the same mathematical problem, since they both lead to the delay between the moment of measurement and the actual action \cite{katsikopoulos2003markov}.
For simplicity, we will focus on the action delay in this paper, and the algorithm and conclusions should be able to generalize to systems with observation delays. We divide the action delay into two main parts into action selection and action actuation. 
For action selection, 
the time length depends on the complexity of the algorithm and the computing power of the processor. System users can limit the action selection time by constraining the searching depth, as in AlphaGo \cite{silver2016mastering}. 
For action actuation, on the other hand, the actuators (e.g., motors, hydraulic machines) also need time to respond to the selected action. For instance, it usually takes more than 0.4 seconds for the hydraulic automotive brake system to generate the desired deceleration \cite{bayan2009brake}. The actuation delay is usually decided by the hardware.

To formulate a delayed system into a DA-MDP, we must select a proper time step for discretely updating the environment. As shown in Fig.~\ref{fig:dmdpn}, the action selected at the current time step $\augm a_t$ will be encapsulated in $\x_{t+1}$.
Thus, $\augm a_t$ must be accessible at time $t+1$ since the agent needs it as the state, which requires the action selection delay to be at most one time step.
We satisfy this requirement by making the time step of the DA-MDP larger than the action selection duration.

The above definition of DA-MDP assumes that the delay time of the agent is an integer multiple of the time step of the system, which is usually not true for many real-world tasks like robotic control. For that, Schuitema \textit{et al.} \cite{schuitema2010control} has proposed an approximation approach by assuming a \textit{virtual} effective action at each discrete system time step, which could achieve first-order equivalence in linearizable systems with arbitrary delay time. With this approximation, the above DA-MDP structure can be adapted to systems with arbitrary-value delays.

\section{Delay-Aware Markov Reward Process}
\label{sec:del}
Our first step is to show that an MDP with multi-step action delays can be converted to a regular MDP problem by state augmentation. We prove the equivalence of these two by comparing their corresponding Markov Reward Processes (MRPs).
The delay-free MRP is:

\begin{definition} \label{def:MRP}
A Markov Reward Process $(\mathcal{S}, \rho, \kappa, \bar r) = M\!R\!P(M\!D\!P(E), \pi)$ can be recovered from a Markov Decision Process $M\!D\!P(E) = (\mathcal{S}, \mathcal{A}, \rho, p, r)$ with a policy $\pi$, such that
\begin{equation*} \label{mdp_kernel1}
\kappa(s_{t+1} | s_t) = \int_\mathcal{A} p(s_{t+1} | s_t, a) \pi(a | s_t) \ d a,
\end{equation*}
\begin{equation*} \label{mdp_kernel}
\bar r(s_t) = \int_\mathcal{A} r(s_t, a) \pi(a | s_t) \ d a,
\end{equation*}
where $\kappa$ is the sate transition distribution and $\bar r$ is the state reward function of the MRP. $E$ is the original environment without delays.
\end{definition}
In the delay-free framework, at each time step, the agent selects an action based on the current observation. The action will immediately be executed in the environment to generate the next observation.
However, if an action delay exists, the interaction manner between the environment and the agent changes, and a different MRP is generated.
An illustration of the delayed interaction between agents and the environment is shown in Fig.~\ref{fig:acbuf}. The agent interacts with the environment not directly but through an action buffer.

\begin{figure}[h]
  \centering
  \includegraphics[width=0.8\linewidth]{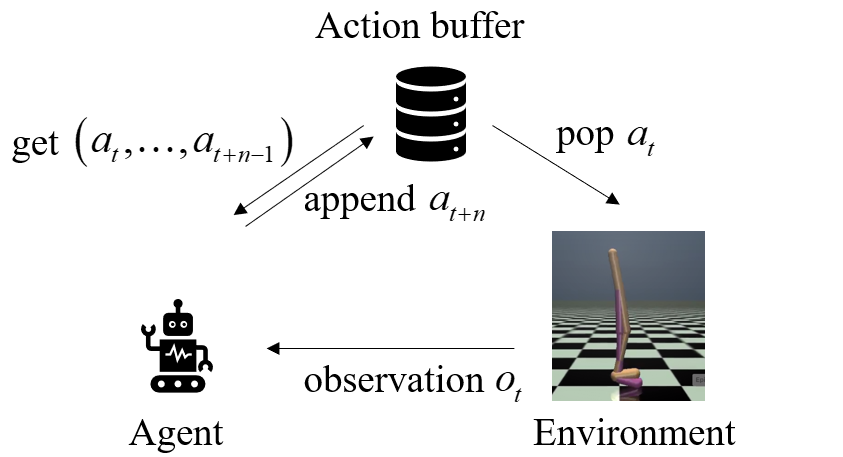}
  \caption{Interaction manner between a delayed agents and the environment. The agent interacts with the environment not directly but through an action buffer. At time $t$, the agent get the observation $o_t$ from the environment as well as a future action sequences $(a_t, \dots, a_{t+n-1})$ from the action buffer. The agents then decide their future action $a_{t+n}$ and store them in the action buffer. The action buffer then pops actions $a_t$ to be executed to the environment.}
  \label{fig:acbuf}
\end{figure}

Based on the delayed interaction manner between the agent and the environment, the Delay-Aware MRP (DA-MRP) is defined as below.

\begin{definition} \label{def:DAMRP}
A Delay-Aware Markov Reward Process $(\pmb{\mathcal{X}}, \augm \rho, \pmb{\kappa}, \pmb{\bar r}) = D\!A\!M\!R\!P(M\!D\!P(E), \augm \pi, n)$ can be recovered from a Markov Decision Process $M\!D\!P(E) = (\mathcal{S}, \mathcal{A}, \rho, p, r)$ with a policy $\augm \pi$ and $n$-step action delay, such that \\
(1) state space
\begin{equation*}
    \pmb{\mathcal{X}} =\mathcal{S} \times \mathcal{A}^{n},
\end{equation*}
(2) initial state distribution
\begin{equation*}
\begin{aligned}
\augm \rho(\x_0) = \augm \rho({s_0, a_0, \dots, a_{n-1}}) = \rho(s_0) \ \prod_{i=0}^{n-1}\delta(a_i - c_i)
\end{aligned}
\end{equation*}
where $(c_i)_{i=1:n-1}$ denotes the initial action sequence, \\
(3) state transition distribution 
\begin{equation*}
\begin{aligned}
&\augm \kappa\left(\x_{t+1}|\x_t\right)  \\
&= \augm \kappa({s_{t+1}, a_{t+1}^{(t+1)}, \dots, a_{t+n}^{(t+1)}} | {s_t, a_t^{(t)}, \dots, a_{t+n-1}^{(t)}}) \\
&= p(s_{t+1} | s_t, a_t) \prod_{i=1}^{n-1}\delta(a_{t+i}^{(t+1)} - a_{t+i}^{(t)}) \augm \pi(a_{t+n}^{(t+1)} |\x_t) 
,
\end{aligned}
\end{equation*}
(4) state-reward function 
\begin{equation*}
\pmb{\bar r}({\x_t}) = \pmb{\bar r}({s_t, a_t, \dots, a_{t+n-1}}) = r(s_t, a_t),
\end{equation*}
\end{definition}

With Def.~\ref{def:MDP}-~\ref{def:DAMRP}, we are ready to prove that DA-MDP is a correct augmentation of MDP with delay, as stated in Theorem.~\ref{the:equn}.

\begin{restatable}{theorem}{TheoremRttb2}
\label{the:equn}
A policy $\augm \pi: \mathcal{A} \times \pmb{\mathcal{X}} \to \mathbb R$ interacting with $D\!A\!M\!D\!P(E, n)$ in the delay-free manner produces the same Markov Reward Process as $\augm \pi$ interacting with $M\!D\!P(E)$ with n-step action delays, i.e.
\begin{equation} \label{DMDP_equality}
    D\!A\!M\!R\!P(M\!D\!P(E), \augm \pi, n) = M\!R\!P(D\!A\!M\!D\!P(E, n), \augm \pi).
\end{equation}
%
\end{restatable}

\begin{proof} For any $M\!D\!P(E)=(\mathcal{S}, \mathcal{A}, \rho, p, r)$, we need to prove that the above two MRPs are the same. Referring to Def.~\ref{def:DMDP}~and~\ref{def:MRP}, for $M\!R\!P(D\!M\!D\!P(E, n), \augm \pi)$, we have \\
(1) state space \hspace{0.3cm} $\mathcal{S} \times \mathcal{A}^{n}$,\\
(2) initial distribution 
\begin{align*}
\augm \rho(\x_0) &= \augm \rho({s_0, a_0, \dots, a_{n-1}})\\
&= \rho(s_0) \prod_{i=0}^{n-1}\delta(a_i - c_i),
\end{align*}
(3)  transition kernel
\begin{align*}
\kappa(\x_{t+1} | \x_t) =& \int_\mathcal{A} \augm p(\x_{t+1} | \x_t, \augm a_t) \augm \pi(\augm a | \x_t) \ d \augm a \\
=&\int_\mathcal{A} p(s_{t+1} | s_t, a_t) \prod_{i=1}^{n-1}\delta(a_{t+i}^{(t+1)} - a_{t+i}^{(t)}) \\
&\delta(a_{t+n}^{(t+1)} - \augm a) \ \augm\pi(\augm a | x_t) \ d \augm a  \\
=& p(s_{t+1} | s_t, a_t) \ \prod_{i=1}^{n-1}\delta(a_{t+i}^{(t+1)} - a_{t+i}^{(t)}) \augm \pi(a_{t+n}^{(t+1)} | \x_t), 
\end{align*}
(4) state-reward function
\begin{align*}
\bar r(\x_t) &= \int_\mathcal{A} r(\x_t, \augm a) \augm \pi(\augm a | \x_t) \ d \augm a  \\
&=\int_\mathcal{A} r(s_t, a_t) \ \augm\pi(\augm a | \x_t) \ d \augm a \\
&=r(s_t, a_t) \int_A \pi(\augm a | \x_t) \ d \augm a \\
&=r(s_t, a_t).
\end{align*}
Since the expanded terms of $M\!R\!P(D\!M\!G(E, n), \augm \pi)$ match the corresponding terms of $D\!A\!M\!R\!P(M\!G(E), \augm \pi, n)$ (Def.~\ref{def:DAMRP}), Eq.~\ref{DMDP_equality} holds.
\end{proof}

\section{Delay-Aware Model-Based Reinforcement Learning}
\label{sec:dambrl}
Theorem.~\ref{the:equn} shows that instead of solving MDPs with action delays, we can alternatively solve the corresponding DA-MDPs.
From the transition function of a $D\!A\!M\!D\!P(E, n)$ with multi-step delays 
\begin{equation} \label{DMDP_dynamics}
\begin{aligned}
&\augm p(\x_{t+1}|\x_t, \augm a_t) = \\
&p(s_{t+1} | s_t, a_t) \prod_{i=1}^{n-1}\delta(a_{t+i}^{(t+1)} - a_{t+i}^{(t)}) \delta(a_{t+n}^{(t+1)} - \augm a_t),
\end{aligned}
\end{equation}
we see that the dynamics is divided into the unknown original dynamics $p(s_{t+1} | s_t, a_t)$ and the known dynamics $\prod_{i=1}^{n}\delta(a_{t+i}^{(t+1)} - a_{t+i}^{(t)}) \delta(a_{t+n} - \augm a_t)$ caused by the action delays. 
Thus, solving DA-MDPs with standard reinforcement learning algorithms will suffer from the curse of dimensionality if assuming a completely unknown environment. In this section, we propose a delay-aware model-based reinforcement learning framework to achieve high computational efficiency.

As mentioned, RTAC \cite{ramstedt2019real} has been proposed to deal the delay problem. However, we will show that this method is only efficient for 1-step delay. When $n=1$ for $D\!M\!D\!P(E, n)$, any transition $\left(s_t, a_t, s_{t+1}\right)$ in the replay buffer is always a valid transition in the Bellman equation with the state-value function as
\begin{equation*} \label{EqV}
\begin{aligned}
&v_\text{\textit{DA-MDP(E, n)}}^{\augm \pi}\left(\x_t\right) = r\left(s_t, a_t\right) \\
&+\E_{s_{t+1}} \left[\E_{\augm a_t} \left[ v_\text{\textit{DA-MDP(E, n)}}^{\augm \pi}\left({s_{t+1},a_{t+1}, \dots, a_{t+n-1} ,\augm a_{t}}\right) \right] \right],
\end{aligned}
\end{equation*}
where $\augm a_t \sim \augm \pi(\cdot | \x_t)$, and $s_{t+1} \sim p(\cdot | s_t, a_t)$. 
However, when considering the multi-step delay, i.e., $n\geq2$, it is challenging to use off-policy model-free reinforcement learning because augmented transitions need to be stored and we only learn the effect of an action on the state-value function after $n$-step updates of the Bellman equation. 
Also, the dimension of the state vector $\x$ increases with the delay step $n$, resulting in the exponential growth of the state-space.

Another limitation of model-free methods for DA-MDPs is that it can be difficult to transfer the learned knowledge (e.g., value functions, policies) when the action delay step $n$ changes because the input dimensions of the value functions and policies depend on the delay step $n$. The agent must learn again from scratch whenever the system delay changes, which is usual in real-world systems. 

The problems of model-free methods have motivated us to develop model-based reinforcement learning (MBRL) methods to combat the action delay. MBRL tries to solve MDPs by learning the dynamics model of the environment. Intuitively, we can inject our knowledge into the learned model without leaning effort.
Based on this intuition, in this paper, we propose a delay-aware MBRL framework to solve multi-step DA-MDPs which can efficiently alleviate the aforementioned two problems of model-free methods.
From Eq.~\ref{DMDP_dynamics}, the unknown part is exactly the dynamics that we learn in MBRL algorithms for delay-free MDPs. In our proposed framework, only $p(s_{t+1} | s_t, a_t)$ is learned and the dynamics caused by the delay is combined with the learned model by adding action delays to the interaction scheme. As mentioned, the learned dynamics model is transferable between systems with different delay steps, since we can adjust the interaction scheme based on the delay step 
(See Section~\ref{sec:trans} for an explanation of the transfer performance).

The proposed framework of delay-aware MBRL is shown in Algorithm~\ref{alg:rtmbrl}. In the \textbf{for} loop, we are solving a planning problem, given a dynamics model with an initial action sequence. For that, the learned model is used not only for the optimal control but also for the state prediction to compensate for the delay effect. By iteratively training, we gradually improve the model accuracy and obtain better planning performance and , especially in high-reward regions.

\begin{algorithm}[!t]
   \caption{Delay-Aware Model-Based Reinforcement Learning}
   \label{alg:rtmbrl}
\begin{algorithmic}
    \STATE {\bfseries Input:} action delay step $n$, initial actions $(a_i)_{i=0,\dots, n-1}$, and task horizon $T$
    \STATE {\bfseries Output:} learned transition probability $\widetilde p$
    \floatname{algorithm}{Procedure}
     \STATE{Initialize replay buffer $\mathbb{D}$ with a random controller for one trial.}
     \FOR {Episode $k = 1$ to $K$}
     \STATE{Train a dynamics model $\widetilde p$ given $\mathbb{D}$.}
      \STATE{Optimize action sequence \smash{$a_{n+1:T}$} with initial actions $(a_i)_{i=0,\dots, n-1}$ and  estimated system dynamics $\widetilde p$}
     \STATE{Record experience: \smash{$\mathbb{D} \leftarrow \mathbb{D} \cup (s_{t}, a_{t}, s_{t+1})_{t=0:T}$}.}
     \ENDFOR
\end{algorithmic}
\end{algorithm}

\subsection{Delay-Aware Trajectory Sampling}
Recently, several MBRL algorithms have been proposed to match the asymptotic performance of model-free algorithms on challenging benchmark tasks, including probabilistic ensemble with trajectory sampling (PETS) \cite{chua2018deep}, model-based policy optimization (MBPO) \cite{janner2019trust}, model-based planning with policy networks (POPLIN) \cite{wang2019exploring}, etc. In this section, we will combine the state-of-the-art PETS algorithm with the proposed delay-aware MBRL framework to generate a new method for solving DA-MDPs. We name the method as the Delay-Aware Trajectory Sampling (DATS). 

In DATS, the dynamic model is represented by an ensemble of probabilistic neural networks that output Gaussian distributions which helps model the aleatoric uncertainty. The use of the ensemble can help incorporate the epistemic uncertainty of the dynamic model and approximate the Bayesian posterior \cite{osband2016deep, lakshminarayanan2017simple}. The planning of action sequences applies the concept of model predictive control (MPC) with the cross-entropy method (CEM) for elite selection of the sampled action sequences. In the most inner \textbf{for} loop of Algorithm~\ref{alg:rtpets}, with the current state $s_t$, we first propagate state particles with the same action sequence $a_{t:t+n-1}$ to make various estimates of the future state $s_{t+n}$ 
, and then use sampled action sequences $a_{t+n:t+n+m}$ to predict $s_{t+n+1:t+n+1+m}$ for each particle. In this way, the uncertainty of the learned model is considered in both state-prediction and planning phases, which improves the robustness of the algorithm. The complete algorithm is shown in Algorithm~\ref{alg:rtpets}. 

\begin{algorithm}[!t]
   \caption{Delay-Aware Trajectory Sampling}
   \label{alg:rtpets}
\begin{algorithmic}
    \STATE {\bfseries Input:} action delay step $n$, initial actions $(a_i)_{i=0,\dots, n-1}$, task horizon $T$, planning horizon $m$
    \STATE {\bfseries Output:} learned transition probability $\widetilde p$
    \floatname{algorithm}{Procedure}
     \STATE{Initialize transition buffer $\mathbb{D}$ with a random controller for one trial.}
     \FOR {Trial $k = 1$ to $K$}
     \STATE{Train a probabilistic dynamics model $\widetilde p$ given $\mathbb{D}$.}
     \STATE {Initialize action buffer $\mathbb{A} = (a_i)_{i=0,\dots, n-1}$}
     \FOR {Time $t = 0$ to $T-n$}

     \STATE{Observe $s_{t}$}
     \FOR {Sampled \smash{$a_{t+n:t+n+m}\!\sim\!\text{CEM}(\cdot)$}}
     \STATE {Concatenate $a_{t+n:t+n+m}$ with $a_{t:t+n-1}$}
     \STATE{Propagate state particles \smash{$s_{\tau}$} using $\widetilde p$.}
     \STATE{Evaluate actions as \smash{$\sum_{\tau=t}^{t+n+m}{{r(s_{\tau}, a_{\tau})}}$}}
     \STATE{Update \smash{$\text{CEM}(\cdot)$} distribution.}
     \ENDFOR
      \STATE{Pick the first action \smash{${a_{t+n}}$} from optimal action sequence and store in $\mathbb{A}$ }
     \ENDFOR
     \STATE{Record experience: \smash{$\mathbb{D} \leftarrow \mathbb{D} \cup (s_{t}, {a_{t}}, s_{t+1})_{t=0:T}$}.}
     \ENDFOR
\end{algorithmic}
\end{algorithm}

Model-based methods have a natural advantage when dealing with multi-step DA-MDPs when compared with model-free methods. 
With model-free methods, the effect of an action on the state-value function can only be learned after $n$-time updates of the Bellman equation. The agent implicitly wastes both time and effort to learn the known part of system dynamics caused by action delay since it does not understand the meaning of the elements in the state vectors.
As mentioned, the advantage of model-based methods is that they incorporate delay effect into the system dynamics without extra learning 
(see Section~\ref{sec:vs} for a performance comparison between model-free and model-based methods).

\begin{figure}[t]
\centering
\begin{subfigure}[t]{0.2\textwidth}
\includegraphics[height=0.8\textwidth]{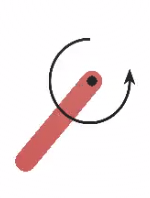} 
\caption{Pendulum}
\end{subfigure}
\begin{subfigure}[t]{0.2\textwidth}
\includegraphics[height=0.8\textwidth]{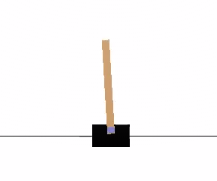}
\caption{CartPole}
\end{subfigure}
\begin{subfigure}[t]{0.2\textwidth}
\includegraphics[height=0.8\textwidth]{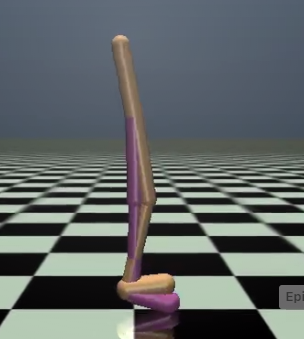}
\caption{Walker2d}
\end{subfigure}
 \begin{subfigure}[t]{0.2\textwidth}
\includegraphics[height=0.8\textwidth]{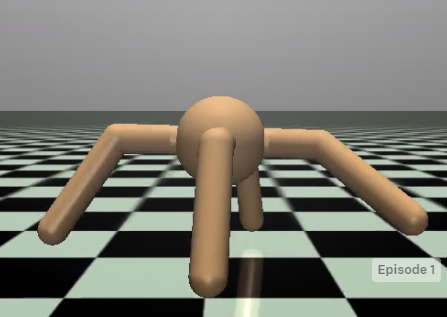}
\caption{Ant}
\end{subfigure}
\caption{Benchmark environments.}
\label{fig:envs}
\end{figure}

\begin{figure*}[!t]
\vskip 0.2in
\centering
\begin{subfigure}{0.4\textwidth}
\includegraphics[width=\linewidth]{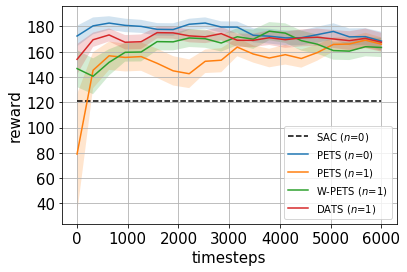} 
\caption{Pendulum-v0}
\label{fig:cart}
\end{subfigure}
\begin{subfigure}{0.4\textwidth}
\includegraphics[width=\textwidth]{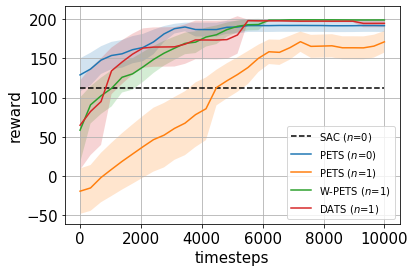}
\caption{CartPole-v1}
\label{fig:rea}
\end{subfigure}
\begin{subfigure}{0.4\textwidth}
\includegraphics[width=\linewidth]{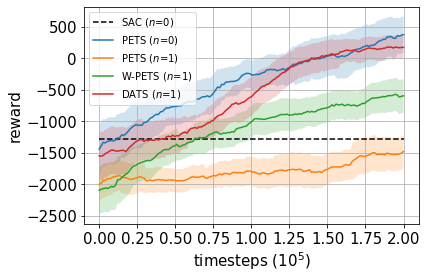}
\caption{Walker2d-v1}
\label{fig:wal}
\end{subfigure}
 \begin{subfigure}{0.4\textwidth}
\includegraphics[width=\linewidth]{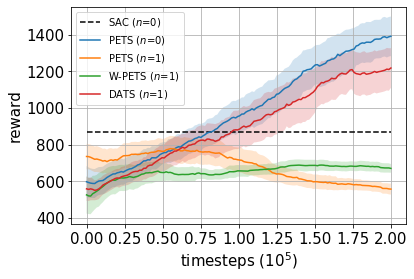}
\caption{Ant-v1}
\label{fig:ant}
\end{subfigure}
\caption{Performances (means and standard deviations of rewards) of different MBRL algorithms in Gym environments. The environment is non-delayed for SAC and PETS ($n=0$) and is one-step-delayed for other algorithms. DATS is the proposed algorithm. The results indicate that the performance degradation resulting from the environment action delay is minimal when using DATS is minimal.}
\label{fig:performance}
\vskip -0.2in
\end{figure*}

\section{Experiments}
\subsection{Reinforcement Learning in Delayed Systems}
\label{sec:exp}


Experiments are conducted across four OpenAI Gym/Mujoco \cite{brockman2016openai, todorov2012mujoco} environments for continuous control: $\mathtt{Pendulum}$, $\mathtt{Cartpole}$, $\mathtt{Walker2d}$ and $\mathtt{Ant}$ as shown in Fig.~\ref{fig:envs}. The details of the environments are described below.

$\mathtt{Pendulum.}$ A single-linked pendulum is fixed on the one end, with an actuator located on the
joint. In this version of the problem, the pendulum starts in a random position, and the goal is to swing it up to keep it upright.
Observations include the joint angle and the joint angular velocity. The reward penalizes position and speed deviations from the upright equilibrium and the magnitude of the control input.

$\mathtt{Cartpole.}$ A pole is connected to the cart through an un-actuated joint, and the cart moves along a frictionless track. Control the system by applying a real-number force to the cart. The pole starts upright, and the goal is to prevent it from falling over. Let $\theta_t$ be the angle of the pole away from the upright
vertical position, and $x_t$ be the position where the cart leaves the center of the rail at time $t$. The
4-dimensional observation at time $t$ is ($x_t$, $\theta_t$, $\dot{x}_t$, $\dot{\theta}_t$). A reward of +1 is provided for every timestep that the pole remains upright.

$\mathtt{Walker2d.}$ Walker2d is a 2-dimensional bipedal robot, consisting of 7 rigid links, including a torso and 2 legs.
There are 6 actuators, 3 for each leg. The observations include the (angular) position and speed of
all joints. The reward is the $x$ direction speed plus the penalty for the distance to a target height and the magnitude of control input. The goal is to walk forward as
fast as possible while keeping the standing height with minimal control input.

$\mathtt{Ant.}$ Ant is a 3-dimensional 4-legged robot with 13 rigid links (including a torso and 4 legs). There are 8 actuators at the joints, 2 for
each leg. The observations include the (angular) position and speed of all
joints. The reward is the $x$ direction speed plus
penalty for the distance to a target height and the magnitude of control input. The goal is to walk forward as fast as possible, and
approximately maintain the normal standing height with minimal control input.

Among the 4 continuous control tasks, the tasks of $\mathtt{Walker2d}$ and $\mathtt{Ant}$ are considered more challenging than $\mathtt{Pendulum}$ and $\mathtt{Cartpole}$ indicated by the dimension of dynamics. 

In experiments, we add delays manually by revising the interaction framework between the agents and the environments if needed.

To show the advantage of DATS, we use 5 algorithms:
\begin{itemize}
    \item \textbf{SAC} ($n=0$): Soft actor-critic \cite{haarnoja2018soft} is a state-of-the-art model-free reinforcement learning algorithm serving as a model-free baseline. Only the performances at the maximum time step are visualized.
    \item \textbf{PETS} ($n=0$): The PETS algorithm \cite{chua2018deep} is implemented in the non-delayed environment without action delays, providing the performance upper bound for algorithms in delayed environments.
    \item \textbf{PETS} ($n=1$): The PETS algorithm is blindly implemented in the 1-step delayed environment without modeling action delays.
    \item \textbf{W-PETS} ($n=1$): The PETS algorithm is augmented to solve DA-MDPs with $n=1$. However, it inefficiently tries to learn the \textbf{whole} dynamics $p(\x_{t+1}|\x_t, \augm a_t)$ as shown in Eq.~\ref{DMDP_dynamics} including the known part caused by actions delays.
    \item \textbf{DATS} ($n=1$): DATS is our proposed method as in Algorithm~\ref{alg:rtpets}. It incorporates the action delay into the framework and only learns the unknown original dynamics $p(s_{t+1}|s_t, a_t)$ as shown in Eq.~\ref{DMDP_dynamics}.

\end{itemize}

Each algorithm is run with 10 random seeds in each environment. 
Fig.~\ref{fig:performance} shows the algorithmic performances. 
As the model-free baseline, SAC is not as efficient as PETS in the four environments when there are no delays.
While PETS ($n=1$) has the worst performance because the agent does not consider the action delay and learns the wrong dynamics, it can still make some improvements in simple environments like $\mathtt{Pendulum}$ (Fig.~\ref{fig:cart}), and $\mathtt{Cartpole}$ (Fig.~\ref{fig:rea}) due to the correlation of transitions. 
PETS ($n=1$) performs poorly for tasks that need accurate transition dynamics for planning in $\mathtt{Walker2d}$ (Fig.~\ref{fig:wal}) and $\mathtt{Ant}$ (Fig.~\ref{fig:ant}).
W-PETS achieves similar performance with PETS in $\mathtt{Pendulum}$ and $\mathtt{Cartpole}$.
But its performance also degrades a lot when the task gets more difficult since it has to learn the dynamics of the extra $n$ dimensions of states caused by the $n$-step action delays (Fig.~\ref{fig:wal} and \ref{fig:ant}).
DATS performs the same as PETS for the four tasks, i.e., action delays do not affect DATS.

\begin{figure*}[!t]
\vskip 0.2in
\centering
\begin{subfigure}{0.4\textwidth}
\includegraphics[width=\linewidth]{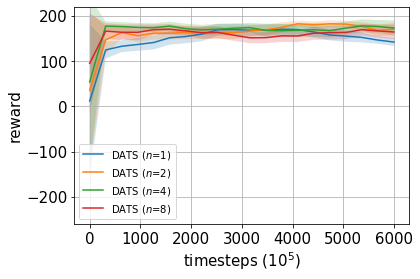} 
\caption{DATS in Pendulum-v0}
\label{fig:rtpetsc}
\end{subfigure}
\begin{subfigure}{0.4\textwidth}
\includegraphics[width=\textwidth]{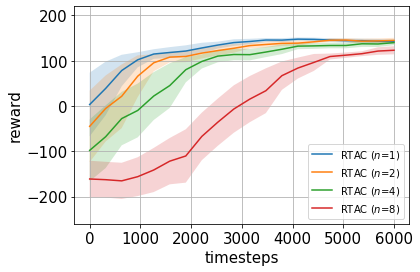}
\caption{RTAC in Pendulum-v0}
\label{fig:rtacc}
\end{subfigure}
\begin{subfigure}{0.4\textwidth}
\includegraphics[width=\linewidth]{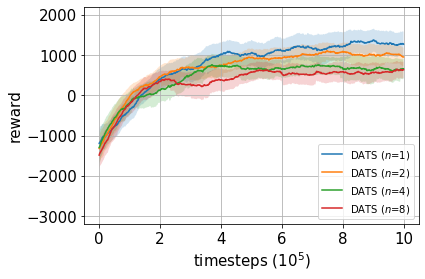}
\caption{DATS in Walker2d-v1}
\label{fig:rtpetsw}
\end{subfigure}
 \begin{subfigure}{0.4\textwidth}
\includegraphics[width=\linewidth]{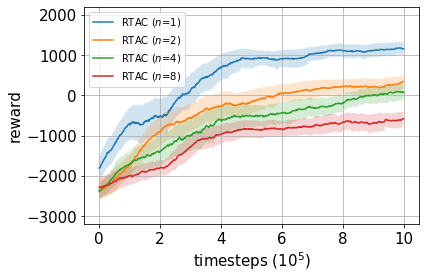}
\caption{RTAC in Walker2d-v1}
\label{fig:rtacw}
\end{subfigure}
\caption{Performances (means and standard deviations of rewards) of DATS and RTAC in Gym environments with different action delay steps. The model-based algorithm DATS outperforms the model-free algorithm RTAC in terms of efficiency and stability. RTAC degrades significantly as the delay step increases. }
\label{fig:mbmf}
\vskip -0.2in
\end{figure*}

The reason why DATS in delayed environment matches the asymptotic performance of PETS in the non-delayed environment is that the quality and quantity of transitions $\left(s_t, a_t, s_{t+1}\right)$ used for model training in DATS are almost the same with PETS, despite the action delay. The slight difference is due to the distribution shift caused by the predefined initial actions, which has minimal influence on the overall performance if the task horizon is long enough compared to the action delay step.

\subsection{Model-Based vs Model-Free}
\label{sec:vs}
To show the advantage of the proposed delay-aware MBRL framework when dealing with multi-step delays,
we compare the model-free algorithm RTAC \cite{ramstedt2019real} and the proposed model-based DATS. RTAC is suitable for solving DA-MDPs and is modified based on SAC, but as explained in Section~\ref{sec:dambrl}, RTAC can avoid extra learning only when the action delay is exactly one-step. 

We test them in the simple environment $\mathtt{Pendulum}$ and the complex environment $\mathtt{Walker2d}$ with various delay step $n$. The learning curves in Fig.~\ref{fig:mbmf}.
show that DATS outperforms RTAC in efficiency and stability. DATS keeps consistent performance while RTAC degrades significantly as the delay step increases, even for the simple task $\mathtt{Pendulum}$, as shown in Fig.~\ref{fig:rtacc}. The reason is that with the original dynamics of $\mathtt{Pendulum}$ and $\mathtt{Walker2d}$ fixed, the extra dynamics caused by the action delay rapidly dominates the dimension of the state space of the learning problem as the delay step increases, and exponentially more transitions are needed to sample and learn.

\begin{table*}[t]
\caption{Reward matrix of DATS and RTAC}
\label{tab:trans}
\centering
\begin{small}
\begin{subtable}{\linewidth}\centering
\caption{Pendulum-v0}
\begin{tabular}{lcccccr}
\toprule
\multirow{2}{*}{$n$} & \multicolumn{4}{c}{DATS} & \multirow{2}{*}{RTAC}\\
& $\widetilde p_1$ & $\widetilde p_2$ & $\widetilde p_4$  &$\widetilde p_8$ & \\
\midrule
1   &154.10$\pm$14.86& 156.37$\pm$13.29& \textbf{163.29$\pm$16.03} & 149.78$\pm$13.778&121.36$\pm$12.63\\
2 &\textbf{163.92$\pm$15.23}& 162.93$\pm$14.26& 155.90$\pm$16.11 & 160.07$\pm$18.30 &109.44$\pm$12.58\\
4    &160.39$\pm$12.63& 162.87$\pm$16.21& \textbf{171.53$\pm$10.85} &166.29$\pm$14.22 &80.15$\pm$27.94\\
8    &163.29$\pm$15.53& 151.20$\pm$13.44& 166.37$\pm$13.32 &\textbf{166.59$\pm$10.59} &-110.28$\pm$58.89\\
16    &153.41$\pm$17.35& \textbf{159.09$\pm$19.88} & 153.89$\pm$14.22  & 149.90$\pm$16.86 &-122.98$\pm$64.82\\
\bottomrule
\end{tabular}
\end{subtable}
\vskip 0.1in
\begin{subtable}{\linewidth}\centering
\caption{Walker2d-v1}
\begin{tabular}{lcccccr}
\toprule
\multirow{2}{*}{$n$} & \multicolumn{4}{c}{DATS} & \multirow{2}{*}{RTAC}\\
& $\widetilde p_1$ & $\widetilde p_2$ & $\widetilde p_4$  &$\widetilde p_8$ & \\
\midrule
1   &471.34$\pm$426.26& \textbf{524.76$\pm$387.67}& 496.13$\pm$442.89 & 395.78$\pm$409.98&-471.13$\pm$896.28\\
2 & 549.73$\pm$410.76& 487.32$\pm$334.49& \textbf{527.98$\pm$477.19} & 492.56$\pm$490.01 &-754.42$\pm$722.79\\
4    & \textbf{485.29$\pm$438.98}& 439.23$\pm$529.39& 248.60$\pm$611.82 &552.91$\pm$410.76 &-1252.47$\pm$710.10\\
8    & 356.93$\pm$431.58& 438.82$\pm$563.13& \textbf{482.09$\pm$316.34} &247.97$\pm$595.63 &-1766.85$\pm$404.28\\
16    & 292.38$\pm$521.86& 311.44$\pm$409.80 & \textbf{473.97$\pm$309.81}  & 401.34 $\pm$634.12 &-2173.87$\pm$625.76\\
\bottomrule
\end{tabular}
\end{subtable}
\end{small}

\end{table*}
\subsection{Transferable Knowledge}
\label{sec:trans}

In this section, we show the transferability of the knowledge learned by DATS. 
We first learn several dynamics models \{$\widetilde p_i$\} in $\mathtt{Pendulum}$ and $\mathtt{Walker2d}$ with DATS, where $i=1,2,4,8$ denotes the action delay step during training. The learned models are then tested in environments with $n$-step action delays ($n=1,2,4,8,16$).
We train the dynamics model in each environment with the same amount of transitions $(s_t, a_t, s_{t+1})$: 2,000 for $\mathtt{Pendulum}$ and 200,000 for $\mathtt{Walker2d}$. The planning method and hyper-parameters stay the same as those in Algorithm~\ref{alg:rtpets}. RTAC provides the model-free baseline for each environment. Recall that since RTAC is a model-free algorithm, when changing the delay steps, it must learn from scratch.

The reward matrix in Table~\ref{tab:trans} shows that DATS performs well even when the delay step is twice larger than the maximum step during model-training ($n=16$) for $\mathtt{Pendulum}$ and $\mathtt{Walker2d}$. 
We infer that the learned knowledge (dynamics in this case) is transferable, i.e., when the action delay of the system changes, the estimated dynamics are still useful by simply adjusting the known part of the dynamics caused by the action delay. On the other hand, RTAC performs poorly as the delay step increases since the dimension of the state space grows and the agent has to spend more effort to learn the delay dynamics. Notably, the learned knowledge of model-free methods cannot transfer when the delay step changes.

The results suggest that the transferability of DATS makes it suitable for Sim-to-Real tasks when there are action delays in real systems, and that the delay step during model training does not have to be equal to the delay step in a real system.
Therefore, if the delay steps of the real-world tasks are known and fixed, we can incorporate the delay effect with the original dynamics learned in the delay-free simulator, and obtain highly efficient Sim-to-Real transformations. 

\section{Conclusion}

This paper proposed a general delay-aware MBRL framework which solves multi-step DA-MDPs with high efficiency and transferability.
Our key insight is that the dynamics of DA-MDPs can be divided into two parts: the known part caused by delays, and the unknown part inherited from the original delay-free MDP.
The proposed delay-aware MBRL framework learns the original unknown dynamics and incorporates the known part of the dynamics explicitly.
We also provided an efficient implementation of delay-aware MBRL as DATS by combining a state-of-the-art modeling and planning method, PETS. The experiment results showed that the performance of PETS in instantaneous environments is similarly to the performance of DATS in delayed environments with respect to delay duration. Moreover, the learned dynamics by DATS is transferable when the time of action delay changes, thus making DATS the preferred algorithm for tasks in real-world systems.

\bibliographystyle{IEEEtran}
\bibliography{main}

\end{document}